\documentclass[12pt]{article}

\usepackage[nonatbib]{arxiv}
\usepackage[toc,page]{appendix}

\usepackage{lineno,hyperref}
\usepackage{floatrow}
\usepackage{caption}
\usepackage{subcaption}
\usepackage{color, soul}
\usepackage{amsmath,amsfonts,amssymb,amscd,amsthm,xspace}
\usepackage{graphicx} 
\usepackage{fancyhdr}
\renewcommand{\vec}[1]{\mathbf{#1}}
\newcommand{\sgn}{\mathop{\mathrm{sgn}}}
\usepackage{units}
\usepackage{mathtools}

\usepackage{multirow}
\usepackage{multicol}
\usepackage[normalem]{ulem}
\usepackage{cancel}
\newcommand{\ignore}[1]{}
\newcommand{\bigO}{\mathcal{O}}
\newtheorem{lemma}{Lemma}
\newcommand{\highlightAdd}{black}    
\newcommand{\highlightDel}[1]{} 

\title{Training Multi-layer Spiking Neural Networks using NormAD based Spatio-Temporal Error Backpropagation}

\author{
  Navin Anwani\thanks{Corresponding author} \\
  Department of Electrical Engineering\\
  Indian Institute of Technology Bombay\\
  Mumbai, 400076, India \\
  \texttt{anwaninavin@gmail.com} \\
   \And
  Bipin Rajendran\\
  Department of Electrical and Computer Engineering\\
  New Jersey Institute of Technology\\
  NJ, 07102 USA\\
  \texttt{bipin@njit.edu} \\
}

\begin{document}

\maketitle

\begin{abstract}
Spiking neural networks (SNNs) have garnered a  great amount of interest for supervised and unsupervised learning applications. This paper deals with the problem of training multi-layer feedforward SNNs. The non-linear integrate-and-fire dynamics employed by spiking neurons make it difficult to train SNNs to generate desired spike trains in response to a given input. To tackle this, first the problem of training a multi-layer SNN is formulated as an optimization problem such that its objective function is based on the deviation in membrane potential rather than the spike arrival instants. Then, an optimization method named Normalized Approximate Descent (NormAD), hand-crafted for such non-convex optimization problems, is employed to derive the iterative synaptic weight update rule. Next, it is reformulated to efficiently train multi-layer SNNs, and is shown to be effectively performing spatio-temporal error backpropagation. 
\textcolor{\highlightAdd}{The learning rule is validated by training $2$-layer SNNs to solve a spike based formulation of the XOR problem as well as training $3$-layer SNNs for generic spike based training problems.}
Thus, the new algorithm is a key step towards building deep spiking neural networks capable of 
efficient 
event-triggered learning.
\end{abstract}

\section{Introduction}
T{he} human brain assimilates multi-modal sensory data and uses it to  \emph{learn} and perform complex cognitive tasks such as pattern detection, recognition and completion. This ability is attributed to the dynamics of approximately $10^{11}$ neurons  interconnected through a network of $10^{15}$ synapses in the human brain. This has motivated the study of neural networks in the brain and attempts to mimic their learning and information processing capabilities to create  \emph{smart learning machines}.
Neurons, the fundamental information processing units in brain, communicate with each other by transmitting action potentials or spikes through their synapses. The process of learning in the brain emerges from synaptic plasticity viz., modification of strength of synapses triggered by spiking activity of corresponding neurons.

Spiking neurons are the third generation of artificial neuron models which closely mimic the dynamics of biological neurons. Unlike previous generations, both inputs and the output of a spiking neuron are signals in time. Specifically, these signals are  point processes of spikes in the membrane potential of the neuron, also called  a spike train. Spiking neural networks (SNNs) are computationally more powerful than previous generations of artificial neural networks as they incorporate  temporal dimension to the information representation and processing capabilities of neural networks \cite{maass1997networks,Bohte2004,CROTTY2005}. Owing to the incorporation of temporal dimension, SNNs naturally lend themselves for processing of signals in time such as audio, video, speech, etc. Information can be encoded in spike trains using temporal codes, rate codes or  population codes \cite{BialekSpikes,Gerstner1997,Prescott2008}. Temporal encoding uses exact spike arrival time for information representation and  has far more representational capacity than  rate code or population code \cite{thorpe2001spike}. However, one of the major hurdles in developing temporal encoding based applications of SNNs is the lack of efficient learning algorithms to train them with desired accuracy. 

In recent years, there has been significant progress in development of neuromorphic computing chips, which are specialized hardware implementations that emulate SNN dynamics inspired by the parallel, event-driven operation of the brain. Some notable examples are the TrueNorth chip from IBM~\cite{merolla2014million}, the Zeroth processor from Qualcomm~\cite{Gehlhaar:ASPLOS2014} and the Loihi chip from Intel~\cite{intelLoihi}. Hence, a breakthrough in learning algorithms for SNNs is apt and timely, to complement the progress of neuromorphic computing hardware.

The present success of deep learning based methods can be traced back to the breakthroughs  in learning algorithms for second generation artificial neural networks (ANNs) \cite{hinton2006fast}. As we will discuss in section~\ref{secLitRev}, there has been work on learning algorithms for SNNs in the recent past, but those methods have not found wide acceptance as they suffer from computational inefficiencies and/or lack of reliable and fast convergence. One of the main reasons for unsatisfactory performance of algorithms developed so far is that those efforts have been centered around adapting high-level concepts from learning algorithms for ANNs or from neuroscience and porting them to SNNs. In this work, we utilize properties specific to spiking neurons in order to develop a supervised learning algorithm for temporal encoding applications with spike-induced weight updates.

A supervised learning algorithm named \emph{NormAD}, for single layer SNNs was proposed in \cite{anwani2015normad}. For a spike domain training problem, it was demonstrated to converge at least an order of magnitude faster than the previous state-of-the-art. 
{Recognizing the importance of multi-layer SNNs for supervised learning, in this paper we extend the idea to derive NormAD based supervised learning rule for multi-layer feedforward spiking neural networks. It is a spike-domain analogue of the error backpropagation rule commonly used for ANNs and can be interpreted to be a realization of spatio-temporal error backpropagation.} The derivation comprises of first formulating  the training  problem for a multi-layer feedforward SNN as a non-convex optimization problem. Next, the Normalized Approximate Descent based optimization, introduced in \cite{anwani2015normad}, is employed to obtain an iterative weight adaptation rule. The new learning rule is successfully validated by employing it to train $2$-layer feedforward SNNs for a spike domain formulation of the XOR problem and $3$-layer feedforward SNNs for general spike domain training problems.

%
%
This paper is organized as follows. We begin with  a summary of learning methods for SNNs documented in literature in section~\ref{secLitRev}. 
Section~\ref{secSpikingNeurons} provides a brief introduction to spiking neurons and the mathematical model of Leaky Integrate-and-Fire (LIF) neuron, also setting the notations we use later in  the paper.
Supervised learning problem for feedforward spiking neural networks is discussed in section~\ref{supLearn}, starting with the description of a generic training problem for SNNs. Next we present a brief mathematical description of a feedforward SNN with one hidden layer and formulate the corresponding training problem as an optimization problem.
Then Normalized Approximate Descent based optimization is employed to derive the spatio-temporal error backpropagation rule in section.~\ref{multiLayerNormAD}. Simulation experiments to demonstrate the performance of the new learning rule for some exemplary supervised training problems are discussed in section.~\ref{validate}. Section~\ref{secConclusion} concludes the development with a discussion on directions for future research that can leverage the algorithm developed here towards the goal of realizing event-triggered deep spiking neural networks.

\section{Related Work}
\label{secLitRev}
 {One of the earliest attempts to demonstrate supervised learning with spiking neurons is the SpikeProp algorithm} \cite{bohte2002error}.  {However, it is restricted to single spike learning, thereby limiting its information representation capacity}. SpikeProp was then extended in \cite{booij2005gradient} to neurons firing multiple spikes. In these studies,  the training problem was formulated as an optimization problem with the objective function in terms of the difference between desired and observed spike arrival instants and  gradient descent was used to adjust the weights. However, since spike arrival time is a discontinuous function of the synaptic strengths,  the optimization problem is non-convex and  gradient descent is prone to local minima.

The biologically observed spike time dependent plasticity (STDP) has been used to derive weight update rules for SNNs in \cite{ponulak2010supervised, DL-ReSuMe,paugam2007supervised}. ReSuMe  and DL-ReSuMe took cues from  both STDP as well as the Widrow-Hoff rule to formulate a supervised learning algorithm \cite{ponulak2010supervised,DL-ReSuMe}. Though these algorithms are biologically inspired, the training time necessary to converge is a concern, especially for real-world applications in large networks. The ReSuMe algorithm has been extended to multi-layer feedforward SNNs using backpropagation in \cite{Sporea}.

Another notable spike-domain learning rule is PBSNLR  \cite{xu2013new}, which is an offline learning rule for the spiking perceptron neuron (SPN) model using the perceptron learning rule.  The PSD algorithm \cite{yu2013precise} uses Widrow-Hoff rule to empirically determine an equivalent learning rule for spiking neurons.  {The SPAN rule} \cite{mohemmed2012span}  {converts input and output spike signals into analog signals and then applies the Widrow-Hoff rule to derive a learning algorithm}. \textcolor{black}{Further, it is applicable  to the training of  SNNs with only one layer.} The SWAT algorithm \cite{Wade2010} uses STDP and BCM rule to derive a weight adaptation strategy for SNNs. The Normalized Spiking Error Back-Propagation (NSEBP) method proposed in \cite{NSEBP} is based on approximations of the simplified Spike Response Model for the neuron. The multi-STIP algorithm proposed in   \cite{Lin2016} defines an inner product for spike trains to approximate a learning cost function. \color{black}As opposed to the above approaches which attempt to develop weight update rules for fixed network topologies, there are also some efforts in developing  feed-forward networks based on evolutionary algorithms where new neuronal connections are progressively added and their weights and firing thresholds updated for every class label in the database  \cite{Schliebs2013,SOLTIC2010}. \color{black}

Recently, an algorithm to learn precisely timed spikes using a leaky integrate-and-fire neuron was presented in \cite{MEMMESHEIMER2014925}. The algorithm converges only when a synaptic weight configuration to the given training problem exists, and can not provide a close approximation, if the exact solution does not exist. To overcome this limitation, another algorithm to learn spike sequences with finite precision is also presented in the same paper. It allows a window of width $\epsilon$ around the desired spike instant within which the output spike could arrive and performs training only on the first deviation from such desired behavior. While it mitigates the non-linear accumulation of error due to interaction between output spikes, it also restricts the training to just one discrepancy per iteration. 
Backpropagation for training deep networks of LIF neurons has been presented in \cite{10.3389/fnins.2016.00508}, derived assuming an impulse-shaped post-synaptic current kernel and treating the discontinuities at spike events as noise. It presents remarkable results on MNIST and N-MNIST benchmarks using rate coded outputs, while in the present work we are interested in training multi-layer SNNs with temporally encoded outputs i.e., representing information in the timing of spikes.

Many previous attempts to formulate supervised learning as an optimization problem employ an objective function formulated in terms of the difference between desired and observed spike arrival times \cite{bohte2002error, booij2005gradient, xu2013supervised, florian2012chronotron}. We will see in section~\ref{secSpikingNeurons} that a leaky integrate-and-fire (LIF) neuron can be described as a non-linear spatio-temporal filter, spatial filtering being the weighted summation of the synaptic inputs to obtain the total incoming synaptic current and temporal filtering being the leaky integration of the synaptic current to obtain the membrane potential. Thus, it can be argued that in order to train multi-layer SNNs, we would need to backpropagate error in space as well as in time, and as we will see in section~\ref{multiLayerNormAD}, it is indeed the case for proposed algorithm. Note that while the membrane potential can directly control the output spike timings, it is also relatively more tractable through synaptic inputs and weights compared to spike timing. 
This observation is leveraged to derive a spaio-temporal error backpropagation algorithm by treating supervised learning as an optimization problem, with the objective function formulated in terms of the membrane potential. 

\section{Spiking Neurons}
\label{secSpikingNeurons}
Spiking neurons are simplified models of biological neurons e.g., the Hodgkin-Huxley equations describing the dependence of membrane potential of a neuron on its membrane current and conductivity of ion channels~\cite{hodgkin1952quantitative}. A spiking neuron is modeled as a multi-input system that receives inputs in the form of sequences of spikes, which are then transformed to analog current signals at its input synapses. The synaptic currents are superposed inside the neuron and the result is then transformed by its non-linear integrate-and-fire dynamics to a membrane potential signal with a sequence of stereotyped events in it, called action potentials or spikes.
Despite the continuous-time variations in the membrane potential of a neuron, it communicates with other neurons through the synaptic connections by chemically inducing a particular current signal in the post-synaptic neuron each time it spikes. Hence, the output of a neuron can be completely described by the time sequence of spikes issued by it. This is called \emph{spike based information representation} and is illustrated in Fig.~\ref{FigSpkNeuron}. The output, also called a spike train, is modeled as a point process of spike events.
Though the internal dynamics of an individual neuron is straightforward, a network of neurons can exhibit  complex dynamical behaviors. The processing power of neural networks is attributed to the massively parallel synaptic connections among neurons.

\begin{figure}[h]
	\centering
	\includegraphics[scale=0.38]{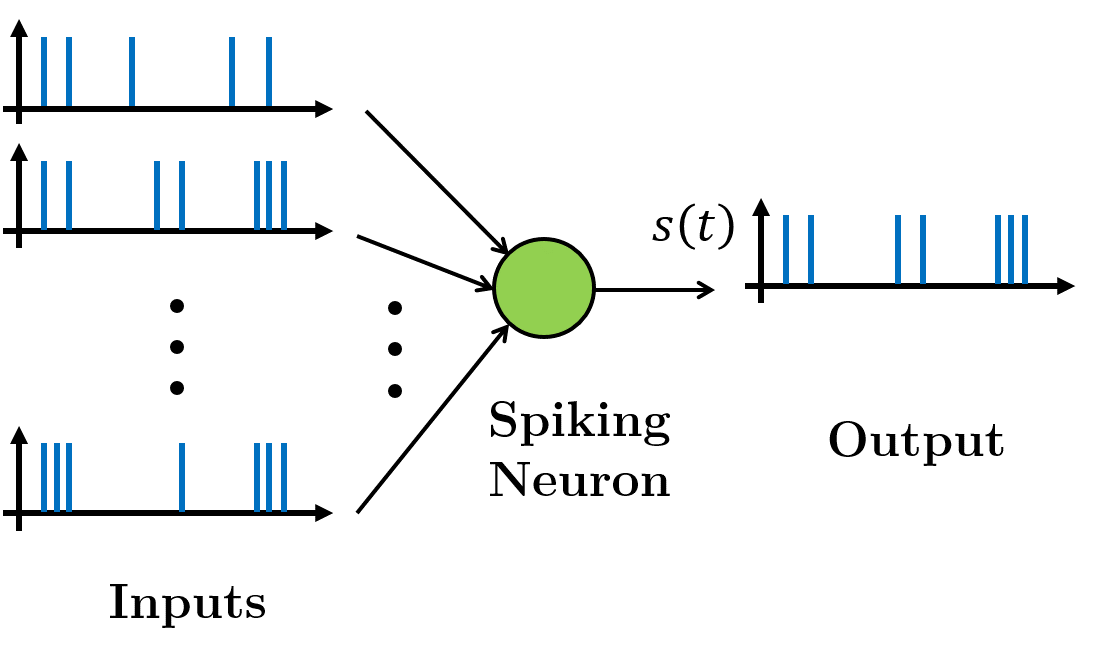}
	\caption{Illustration of spike based information representation: a spiking neuron assimilates multiple input spike trains to generate an output spike train. Figure adapted from \cite{anwani2015normad}.}
	\label{FigSpkNeuron}
\end{figure}

\subsection{Synapse}
\label{secSynapse}
The communication between any two neurons is spike induced and is accomplished through a directed connection between them known as a synapse. In the cortex, each neuron can receive spike-based inputs from thousands of other neurons.
If we model an incoming spike at a synapse as a unit impulse, then the behavior of the synapse to translate it to an analog current signal in the post-synaptic neuron can be modeled by a linear time invariant system with transfer function \(w\alpha(t)\). Thus, if a pre-synaptic neuron issues a spike at time $t^{f}$, the post-synaptic neuron receives a current $i(t)=w\alpha(t-t^{f})$. Here the waveform \(\alpha(t)\) is known as the post-synaptic current kernel and the scaling factor $w$ is called the weight of the synapse. The weight varies from synapse-to-synapse and is representative of its conductance, whereas \(\alpha(t)\) is independent of synapse and is commonly modeled as
\begin{equation}
\label{Eqkernel}
\alpha(t)= \left[ \exp({-t}/{\tau_1})-\exp({-t}/{\tau_2})\right]u(t),
\end{equation}
where $u(t)$ is the  Heaviside step function and $\tau_1 > \tau_2$. Note that the synaptic weight $w$ can be positive or negative, depending on which the synapse is said to be excitatory or inhibitory respectively. Further, we assume that the synaptic currents do not depend on the membrane potential or reversal potential of the post-synaptic neuron.

Let us assume that a neuron receives inputs from $n$ synapses and spikes arrive at the $i^{th}$ synapse at instants $t^i_1, t^i_2, \ldots$. Then, the input signal at the $i^{th}$ synapse (before scaling by synaptic weight $w_{i}$) is given by the expression
\begin{align}
 c_i(t)=\sum_{f}\alpha(t-t^{i}_{f}).
 \label{eqcoft}
\end{align}
The synaptic weights of all input synapses to a neuron are usually represented in a compact form as a weight vector $\vec{w}=\begin{bmatrix}w_1 & w_2 & \cdots & w_n \end{bmatrix}^T$, where $w_i$ is the weight of the $i^{th}$ synapse. The synaptic weights perform spatial filtering over the input signals resulting in an aggregate synaptic current received by the neuron:
\begin{align}
I(t)=\vec{w}^T\vec{c}(t),
\label{eqIoft}
\end{align}
where $ \vec{c}(t)=\begin{bmatrix}c_1(t) & c_2(t) & \cdots & c_n(t)\end{bmatrix}^T$.
A simplified illustration of the role of synaptic transmission in overall spike based information processing by a neuron is shown in Fig.~\ref{Figkernel}, where an incoming spike train at a synaptic input is translated to an analog current with an amplitude depending on weight of the synapse. The resultant current at the neuron  from all its upstream synapses is transformed non-linearly to generate its membrane potential with instances of spikes  viz., sudden surge in membrane potential followed by an immediate drop.

\captionsetup{subrefformat=parens}

\begin{figure}[!h]
	\centering
    \begin{subfigure}[b]{\textwidth}
		\centering
        \includegraphics[scale=0.4]{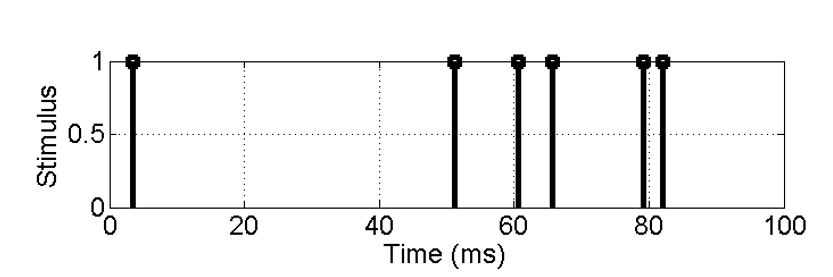}
        \caption{\label{inputSpikes}}
    \end{subfigure}
    \begin{subfigure}[b]{\textwidth}
		\centering
        \includegraphics[scale=0.4]{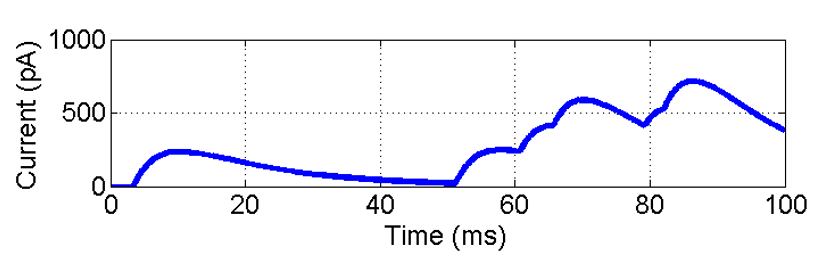}
        \caption{\label{inputCurrent}}
    \end{subfigure}    
    \begin{subfigure}[b]{\textwidth}
		\centering
        \includegraphics[scale=0.4]{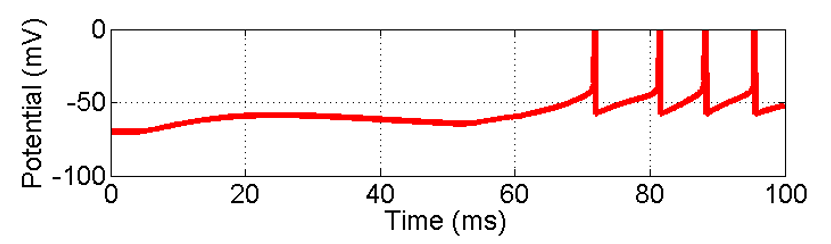}
        \caption{\label{outputSpikes}}
    \end{subfigure}  
        
	\caption{Illustration of a simplified synaptic  transmission and neuronal integration model: \subref{inputSpikes}  exemplary spikes (stimulus) arriving at a synapse, \subref{inputCurrent} the resultant current being fed to the neuron through the synapse and \subref{outputSpikes} the resultant membrane potential of the post-synaptic neuron.}
	\label{Figkernel}
\end{figure}


\subsubsection*{Synaptic Plasticity}
The response of a neuron  to stimuli greatly depends on the conductance of its input synapses. Conductance of a synapse (the synaptic weight) changes based on the spiking activity of the corresponding pre- and post-synaptic neurons. A neural network's ability to learn is attributed to this activity dependent synaptic plasticity. 
Taking cues from biology, we will also constrain the learning algorithm we develop to have spike-induced synaptic weight updates.

\subsection{Leaky Integrate-and-Fire (LIF) Neuron}
\label{secLIF}
In leaky integrate-and-fire (LIF) model of spiking neurons, the transformation from aggregate input synaptic current $I(t)$ to the resultant membrane potential $V(t)$ is governed by the  following differential equation and reset condition \cite{stein1967some}:
\begin{align}
\label{lif1}
C_m & \frac{dV(t)}{dt}=-g_L(V(t)-E_L)+I(t),\\
\nonumber &V(t)\longrightarrow E_L \quad \text{when } V(t)\geq V_T.
\end{align}
Here, $C_m$ is the membrane capacitance, $E_L$ is the leak reversal potential, and $g_L$ is the leak conductance. If $V(t)$ exceeds the threshold potential $V_T$, a spike is said to have been issued at time $t$. The expression $V(t)\longrightarrow E_L$ when $V(t)\geq V_T$ denotes that $V(t)$ is reset to $E_{L}$ when it exceeds the threshold $V_T$. Assuming that the neuron issued its latest spike at time $t_{l}$, Eq.~\eqref{lif1} can be solved for any time instant $t>t_l$, until the issue of the next spike,  with the initial condition $V(t_{l}) = E_{L}$ as
\begin{align}
\label{lif2}
V&(t)=E_L+\left(I(t)u(t-t_{l})\right) \ast h(t)\\
\nonumber &V(t)\longrightarrow E_L \quad \text{when } V(t)\geq V_T,
\end{align}
where `$\ast$' denotes linear convolution and
\begin{align}
h(t)=\frac{1}{C_m}\exp({-t}/{\tau_L})u(t),
\label{eqhoft}
\end{align}
with $\tau_L=C_m/g_L$ is the neuron's leakage time constant.
Note from Eq.~\eqref{lif2} that the aggregate synaptic current $I(t)$ obtained by spatial filtering of all the input signals is first gated with a unit step located at $t=t_{l}$ and then fed to a leaky integrator with impulse response \(h(t)\), which performs temporal filtering. So the LIF neuron acts as a non-linear spatio-temporal filter and the non-linearity is a result of the reset at every spike.

Using Eq.~\eqref{eqIoft} and \eqref{lif2} the membrane potential can be represented in a compact form as
\begin{equation}
V(t)=E_L + \vec{w}^T\vec{d}(t),
\label{lif6}
\end{equation}
where  $\vec{d}(t)=\begin{bmatrix}d_1(t) & d_2(t) & \cdots & d_n(t)\end{bmatrix}^T$ and
\begin{equation}
d_i(t)= \left(c_i(t)u(t-t_{l})\right)*h(t).
\label{eqdoft}
\end{equation}
From Eq.~\eqref{lif6}, it is evident that $\vec{d}(t)$ carries all the information about the input necessary to determine the membrane potential. 
It should be noted that $\vec{d}(t)$ depends on weight vector $\vec{w}$, since $d_i(t)$ for each $i$ depends on  the last spiking instant $t_{l}$, which in turn is dependent on the weight vector $\vec{w}$.

The neuron is said to have spiked only when  the membrane potential $V(t)$ reaches the threshold $V_{T}$. Hence, minor changes in the weight vector $\vec{w}$ may eliminate an already existing spike or introduce new spikes. Thus, spike arrival time $t_{l}$ is a discontinuous function of $\vec{w}$. Therefore, Eq.~\eqref{lif6} implies that $V(t)$ is also discontinuous in weight space.
Supervised learning problem for SNNs is generally framed as an optimization problem with the cost function described in terms of the spike arrival time or  membrane potential. However, the discontinuity of spike arrival time as well as $V(t)$ in weight space renders the cost function discontinuous and hence the optimization problem non-convex. Commonly used steepest descent methods can not be applied to solve such non-convex optimization problems. In this paper, we extend the optimization method named \textit{Normalized Approximate Descent}, introduced in \cite{anwani2015normad} for single layer SNNs  to multi-layer SNNs. 

\subsection{Refractory Period} After issuing a spike, biological neurons can not immediately issue another spike for a short period of time. This short duration of inactivity is called the absolute refractory period ($\Delta_{abs}$). This aspect of spiking neurons has been omitted in the above discussion for simplicity, but can be easily incorporated in our model by replacing $t_{l}$ with $(t_{l} + \Delta_{abs})$ in the equations above. 

Armed with a compact representation of membrane potential in Eq.~\eqref{lif6}, we are now set to derive a synaptic weight update rule to accomplish supervised learning with spiking neurons.

\section{Supervised Learning using Feedforward SNNs}
\label{supLearn}
Supervised learning is the process of obtaining an approximate model of an unknown system based on available training data, where the training data comprises of a set of inputs to the system and corresponding outputs. The learned model should not only fit to the training data well but should also generalize well to unseen samples from the same input distribution. The first requirement viz. to obtain a model so that it best fits the given training data is called training problem. 
Next we discuss the training problem in spike domain, solving which is a stepping stone towards solving the more constrained supervised learning problem.

\begin{figure}[h]
\centering
\includegraphics[scale=0.4]{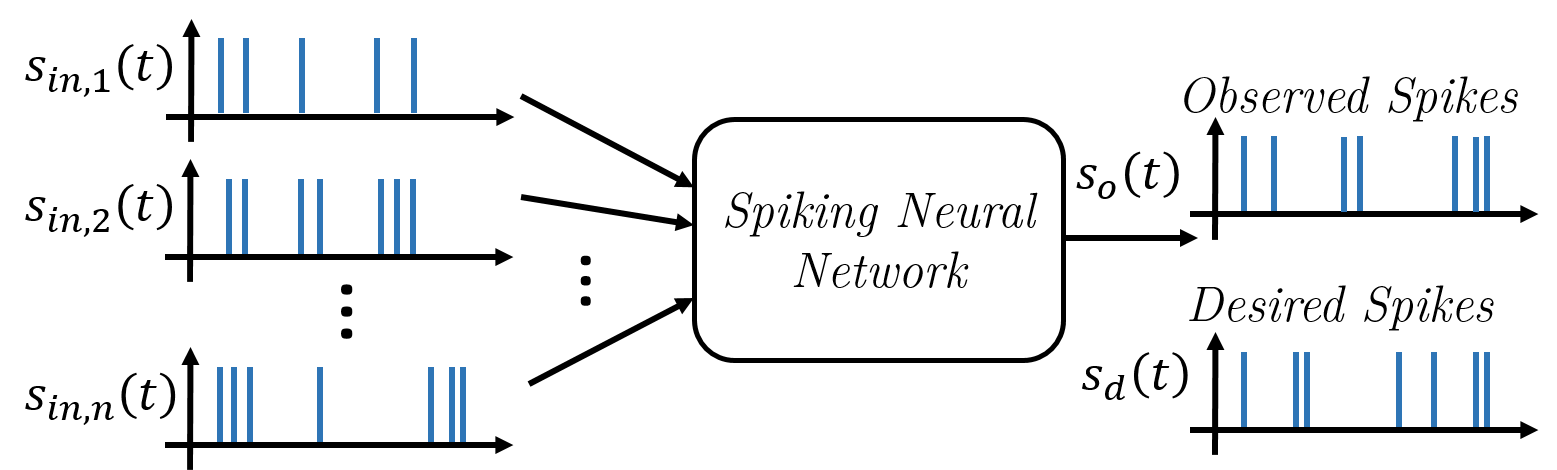}
\caption{Spike domain training problem: Given a set of $n$ input spike trains fed to the SNN through its $n$ inputs, determine the  weights of synaptic connections constituting the SNN so that the observed output spike train is as close as possible to the given desired spike train.}
\label{supLearnfig}
\end{figure}

\subsection{Training Problem}
\label{probTrain}
A canonical training problem for a spiking neural network is illustrated in Fig.~\ref{supLearnfig}. There are $n$ inputs to the network such that $s_{in,i}(t)$ is the spike train fed at the $i^{th}$ input. Let the desired output spike train corresponding to this set of input spike trains be given in the form of an impulse train as
\begin{align}
s_{d}(t)=\sum_{i=1}^f \delta (t-t^{i}_{d}).
\end{align}  
Here, $\delta(t)$ is the Dirac delta function and $ t_{d}^{1},\, t_{d}^{2},\, ...,\, t_{d}^{f} $ are the desired spike arrival instants over a duration $T$, also called an epoch.
The aim is to determine the weights of the synaptic connections constituting the SNN so that its output $s_{o}(t)$ in response to the given input is as close as possible to the desired spike train $s_{d}(t)$.

NormAD based iterative synaptic weight adaptation rule was proposed in \cite{anwani2015normad} for training single layer feedforward SNNs.
However, there are many systems which can not be modeled by any possible configuration of single layer SNN and necessarily require a multi-layer SNN. 
Hence, now we aim to obtain a supervised learning rule for multi-layer spiking neural networks.
The change in weights in a particular iteration of training can be based on the given set of input spike trains, desired output spike train and the corresponding observed output spike train. Also, the weight adaptation rule should be constrained to have spike induced weight updates for computational efficiency. 
For simplicity, we will first derive the weight adaptation rule for training a feedforward SNN with one hidden layer and then state the general weight adaptation rule for feedforward SNN with an arbitrary number of layers.

Imaginary Buffer Line
\subsection*{Performance Metric}
Training performance can be assessed by the correlation between desired and observed outputs. It can be quantified in terms of the cross-correlation between low-pass filtered versions of the two spike trains. The correlation metric which was introduced in \cite{vanRossum} and is commonly used in characterizing the spike based learning efficiency \cite{anwani2015normad,resume} is defined as
\begin{align}
C=\frac{\langle L(s_{d}(t)),L(s_{o}(t)) \rangle}{\|L(s_{d}(t))\| \cdot \|L(s_{o}(t))\|}.
\end{align}
Here, $L(s(t))$ is the low-pass filtered spike train $s(t)$ obtained by convolving it with a one-sided falling exponential i.e.,
\[L(s(t)) = s(t)*(\exp({-t}/{\tau_{LP}})u(t)),\] with $\tau_{LP}=5\,$ms.

\DeclareRobustCommand{\hsout}[1]{\texorpdfstring{\sout{#1}}{#1}}

\begin{figure}[!h]
\centering
\includegraphics[scale=0.45]{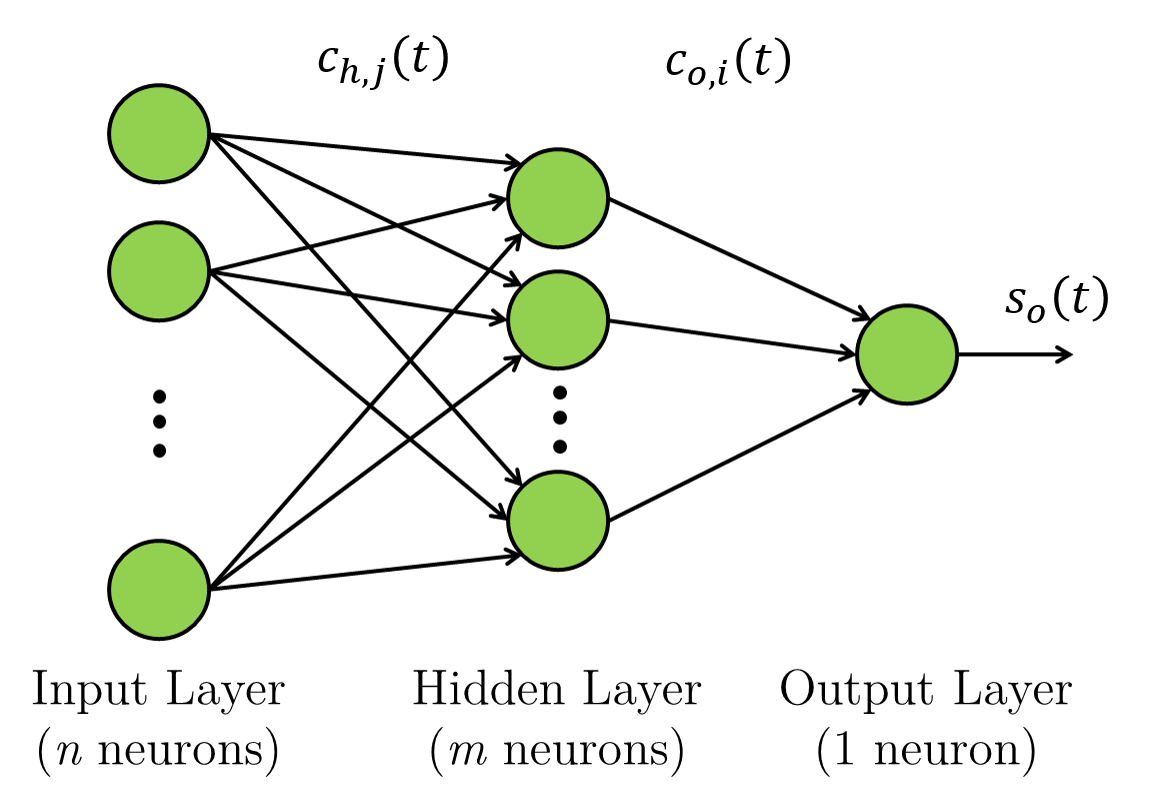}
\caption{Feedforward SNN with one hidden layer ($n \rightarrow m \rightarrow 1$) also known as 2-layer feedforward SNN.}
\label{multiLayerSNN}
\end{figure}

\subsection{Feedforward SNN with One Hidden Layer}
\label{multiLayerIntro}
A fully connected feedforward SNN with one hidden layer is shown in Fig.~\ref{multiLayerSNN}. It has $n$ neurons in the input layer, $m$ neurons in the hidden layer and $1$ in the output layer.
\color{\highlightAdd}
It is also called a 2-layer feedforward SNN, since the neurons in input layer provide spike based encoding of sensory inputs and do not actually implement the neuronal dynamics.
\color{black}
We denote this network as a $n \rightarrow m \rightarrow 1$ feedforward SNN. This basic framework can be extended to the case where there are multiple neurons in the output layer or the case where there are multiple hidden layers. 
The weight of the synapse from the $j^{th}$ neuron in the input layer to the $i^{th}$ neuron in the hidden layer is denoted by $w_{h,ij}$ and that of the synapse from the $i^{th}$ neuron in the hidden layer to the neuron in output layer is denoted by $w_{o,i}$. All input synapses to the $i^{th}$ neuron in the hidden layer can be represented compactly as an $n$-dimensional vector $\vec{w}_{h,i} = \begin{bmatrix}w_{h,i1} & w_{h,i2} & \cdots & w_{h,in}\end{bmatrix}^T$. Similarly input synapses to the output neuron are represented as an $m$-dimensional vector $\vec{w}_o = \begin{bmatrix}w_{o,1} & w_{o,2} & \cdots & w_{o,m}\end{bmatrix}^T$.

Let $s_{in,j}\left(t\right)$ denote the spike train fed by the $j^{th}$ neuron in input layer to neurons in hidden layer.
Hence, from Eq.~\eqref{eqcoft}, the signal fed to the neurons in the hidden layer from the $j^{th}$ input (before scaling by synaptic weight) $c_{h,j}\left(t\right)$ is given as
\begin{align}
c_{h,j}\left(t\right)= s_{in,j}\left(t\right)*\alpha \left(t\right).
\end{align}
Assuming $t_{h,i}^{last}$ as the latest spiking instant of the $i^{th}$ neuron in the hidden layer, define $\vec{d}_{h,i}\left(t\right)$ as
\begin{align}
\label{eqd_hoft}
\vec{d}_{h,i}\left(t\right) = \left(\vec{c}_{h}\left(t\right)u\left(t-t_{h,i}^{last}\right)\right)*h\left(t\right),
\end{align}
where $\vec{c}_{h}\left(t\right)= \begin{bmatrix}c_{h,1} & c_{h,2} & \cdots & c_{h,n}\end{bmatrix}^T$.
From Eq.~\eqref{lif6}, membrane potential of the $i^{th}$ neuron in hidden layer is given as
\begin{align}
V_{h,i}\left(t\right)=E_{L}+\vec{w}_{h,i}^T \vec{d}_{h,i}\left(t\right).
\end{align}
Accordingly, let $s_{h,i}\left(t\right)$ be the spike train produced at the $i^{th}$ neuron in the hidden layer. The corresponding signal fed to the output neuron is given as
\begin{align}
c_{o,i}\left(t\right)= s_{h,i}\left(t\right)*\alpha \left(t\right).
\label{eq_c_oi_of_t}
\end{align}
Defining $\vec{c}_{o}\left(t\right)= \begin{bmatrix}c_{o,1} & c_{o,2} & \cdots & c_{o,m}\end{bmatrix}^T$ and denoting the latest spiking instant of the  output neuron by $t_{o}^{last}$ we can define
\begin{align}
\label{eqd_o_oft}
\vec{d}_{o}\left(t\right) &= \left(\vec{c}_{o}\left(t\right)u\left(t-t_{o}^{last}\right)\right)*h\left(t\right).
\end{align}
Hence, from Eq.~\eqref{lif6}, the membrane potential of the output neuron is given as
\begin{align}
V_{o}\left(t\right)=E_{L}+\vec{w}_{o}^T \vec{d}_{o}\left(t\right)
\end{align}
and the corresponding output spike train is denoted $s_{o}\left( t \right)$.

\subsection{Mathematical Formulation of the Training Problem}
\label{multiLayerProbFormulation}
To solve the training problem employing an $n \rightarrow m \rightarrow 1$ feedforward SNN, effectively we need to determine synaptic weights $W_h = \begin{bmatrix}\vec{w}_{h,1} & \vec{w}_{h,2} & \cdots & \vec{w}_{h,m}\end{bmatrix}^T$ and $\vec{w}_o$ constituting its synaptic connections, so that the output spike train $s_{o}\left(t\right)$ is as close as possible to the desired spike train $s_{d}\left(t\right)$ when the SNN is excited with the given set of input spike trains  $s_{in,i}\left(t\right)$, $i \in \lbrace1,2,...,n \rbrace$. Let $V_{d}\left(t\right)$ be the corresponding ideally desired membrane potential of the output neuron, such that the respective output spike train is $s_{d}\left(t\right)$.
Also, for a particular configuration $W_h$ and $\vec{w}_o$ of synaptic weights of the SNN, let $V_{o}\left(t\right)$ be the observed membrane potential of the output neuron in response to the given input and $s_{o}\left(t\right)$ be the respective output spike train. We define the cost function for training as
\begin{equation}
J\left(W_h, \vec{w}_o\right)=\frac{1}{2} \int_{0}^T \left(\Delta V_{d,o}(t)\right)^2 |e\left(t\right)|dt,
\label{costMulti}
\end{equation}
where
\begin{equation}
\Delta V_{d,o}(t) = V_{d}\left(t\right)-V_{o}\left(t\right)
\label{deltaV_do}
\end{equation}
and 
\begin{equation}
e\left(t\right)= s_{d}\left(t\right)-s_{o}\left(t\right).
\label{eqeoft}
\end{equation}
That is, the cost function is determined by the difference $\Delta V_{d,o}(t)$, only at the instants in time where there is a discrepancy between the desired and observed spike trains of the output neuron.
Thus, the training problem can be expressed as following optimization problem:
\begin{equation}
\begin{aligned}
& \min
& & J\left(W_h, \vec{w}_o\right) \\
& \text{s.t.} 
& & W_h \in \mathbb{R}^{m \times n}, \vec{w}_o \in \mathbb{R}^{m}
\end{aligned}
\label{multilayerOptiProb}
\end{equation}
Note that the optimization with respect to $\vec{w}_o$ is same as training a single layer SNN, provided the spike trains from neurons in the hidden layer  are known. In addition, we  need to derive the weight adaptation rule for synapses feeding the hidden layer viz., the weight matrix $W_{h}$, such that spikes in the hidden layer are most suitable to generate the desired spikes at the output. The cost function is dependent on the membrane potential $V_{o}\left(t\right)$, which is discontinuous with respect to $\vec{w}_o$  as well as $W_{h}$. Hence the optimization problem~\eqref{multilayerOptiProb} is non-convex and susceptible to local minima when solved with steepest descent algorithm.

\section{NormAD based Spatio-Temporal Error Backpropagation}
\label{multiLayerNormAD}
In this section we apply Normalized Approximate Descent to the optimization problem \eqref{multilayerOptiProb} to derive a spike domain analogue of error backpropagation. First we derive the training algorithm for SNNs with single hidden layer, and then we provide its generalized form to train feedforward SNNs with arbitrary number of hidden layers.

\subsection{NormAD -- Normalized Approximate Descent}
Following the approach introduced in \cite{anwani2015normad}, we use  three steps viz., (i) Stochastic Gradient Descent, (ii) Normalization and (iii) Gradient Approximation, as elaborated below to solve the optimization problem \eqref{multilayerOptiProb}.

\subsubsection{Stochastic Gradient Descent}
Instead of trying to minimize the aggregate cost over the epoch, we try to minimize the instantaneous contribution to the cost at each instant $t$ for which $e(t) \neq 0$, independent of that at any other instant and expect that it minimizes the total cost $J\left(W_{h}, \vec{w}_o \right)$.
The instantaneous contribution to the cost at time $t$ is denoted as $J\left(W_h, \vec{w}_o, t\right)$ and is obtained by restricting the limits of integral in Eq.~\eqref{costMulti} to an infinitesimally small interval around time $t$:
\begin{equation}
  J\left(W_h, \vec{w}_o , t\right)= 
 \begin{cases}
 \frac{1}{2}\left(\Delta V_{d,o}(t)\right)^2   & \ e\left(t\right) \neq 0 \\
 0 & \text{otherwise.}
 \end{cases}	
\end{equation}
Thus, using stochastic gradient descent, the prescribed change in any weight vector $\vec{w}$ at time $t$ is given as:
\begin{align*}
\Delta \vec{w}(t)
&=
\begin{cases}
- k(t) \cdot \nabla_{\vec{w}}J\left(W_h, \vec{w}_o , t\right) & \ e\left(t\right) \neq 0\\
0 & \text{otherwise.}
\end{cases}
\end{align*}
Here $k(t)$ is a time dependent learning rate. The change aggregated over the epoch is, therefore
\begin{align}
\nonumber \Delta \vec{w}
& =\int_{t=0}^{T}  - k(t)\cdot \nabla_{\vec{w}}J\left(W_h, \vec{w}_o , t\right) \cdot |e(t)| dt \nonumber\\
& =\int_{t=0}^{T}  k(t) \cdot \Delta V_{d,o}(t)\cdot \nabla_{\vec{w}}V_o\left(t\right) \cdot |e(t)| dt.
\label{eqMultilayerDelw}
\end{align}
Minimizing the instantaneous cost only for time instants when $e(t)\neq 0$ also renders the weight updates spike-induced i.e., it is non-zero only when there is either an observed or a desired spike in the output neuron.

\subsubsection{Normalization}
Observe that in Eq.~\eqref{eqMultilayerDelw}, the gradient of membrane potential $\nabla_{\vec{w}}V_o(t)$ is scaled with the error term $\Delta V_{d,o}(t)$, which serves two purposes. First, it determines the sign of the weight update at time $t$ and second, it gives more importance to weight updates corresponding to the instants with higher magnitude of error. But $V_{d}(t)$ and hence error $\Delta V_{d,o}(t)$ is not known. Also, dependence of the error on $\vec{w}_{h,i}$ is non-linear, so we eliminate the error term $\Delta V_{d,o}(t)$ for neurons in hidden layer by choosing $k\left(t\right)$ such that
\begin{align}
| k\left(t\right) \cdot \Delta V_{d,o}(t)|=r_h,
\label{eqHiddenkoft}
\end{align}
where $r_h$ is a constant. 
From Eq.~\eqref{eqMultilayerDelw}, we obtain the weight update for the $i^{th}$ neuron in the hidden layer as
\begin{align}
\Delta \vec{w}_{h,i} &= r_h \int_{t=0}^T  \nabla_{\vec{w}_{h,i}}V_o\left(t\right) e\left(t\right) dt,
\label{eqMultilayerNormed}
\end{align}
since $\sgn \left(\Delta V_{d,o}(t)\right) = \sgn \left( e(t)\right)$.
For the output neuron, we eliminate the error term by choosing $k\left(t\right)$ such that
\begin{align*}
\|k(t) \cdot \Delta V_{d,o}(t)\cdot \nabla_{\vec{w}_{o}}V_o\left(t\right)\| = r_{o},
\end{align*}
where $r_o$ is a constant.
From Eq.~\eqref{eqMultilayerDelw}, we get the weight update for the output neuron as
\begin{align}
\Delta \vec{w}_{o} &= r_{o} \int_{t=0}^T  \frac{\nabla_{\vec{w}_{o}}V_o\left(t\right)}{\|\nabla_{\vec{w}_{o}}V_o\left(t\right)\|} e\left(t\right) dt.
\label{eqOutNormed}
\end{align}
Now, we proceed to determine the gradients $\nabla_{\vec{w}_{h,i}}V_o\left(t\right)$ and $\nabla_{\vec{w}_{o}}V_o\left(t\right)$.

\subsubsection{Gradient Approximation}
We use an approximation of $V_o\left(t\right)$ which is affine in $\vec{w}_{o}$  and given as
\begin{align}
\widehat{\vec{d}}_o\left(t\right) &= \vec{c}_{o}\left(t\right)*\widehat{h}\left(t\right) \label{eqApproxVecdo}\\
\Rightarrow V_o\left(t\right)  &\approx  \widehat{V}_o\left(t\right) = E_L + \vec{w}^T_{o} \widehat{\vec{d}}_o\left(t\right),
\label{eqApproxVo}
\end{align}
where $\widehat{h}\left(t\right)= \left(\nicefrac{1}{C_m}\right)\exp\left(-t/\tau_L'\right)u\left(t\right)$ with $\tau_L' \leq \tau_L$. Here, $\tau_L'$ is a hyper-parameter of learning rule that needs to be determined empirically.
Similarly $V_{h,i}\left(t\right)$ can be approximated as
\begin{align}
\widehat{\vec{d}}_{h}\left(t\right) &= \vec{c}_{h}\left(t\right)*\widehat{h}\left(t\right)\\
\Rightarrow V_{h,i}\left(t\right) &\approx \widehat{V}_{h,i}\left(t\right) = E_L + \vec{w}^T_{h,i} \widehat{\vec{d}}_{h}\left(t\right).
\label{eqApproxVhi}
\end{align}
Note that $\widehat{V}_{h,i}\left(t\right)$ and $\widehat{V}_{o}\left(t\right)$ are linear in weight vectors $\vec{w}_{h,i}$ and $\vec{w}_{o}$ respectively of corresponding input synapses. 
From Eq.~\eqref{eqApproxVo}, we approximate $\nabla_{\vec{w}_{o}}V_o\left(t\right)$ as
\begin{align}
\nonumber \nabla_{\vec{w}_{o}}V_o\left(t\right)		&\approx \nabla_{\vec{w}_{o}}\widehat{V}_o\left(t\right) \\						&= \widehat{\vec{d}}_o\left(t\right).
\label{eqGradWoVo}
\end{align}
Similarly $\nabla_{\vec{w}_{h,i}}V_o\left(t\right)$ can be approximated as
\begin{align}
\nonumber \nabla_{\vec{w}_{h,i}}V_o\left(t\right)		&\approx \nabla_{\vec{w}_{h,i}}\widehat{V}_o\left(t\right) \\
											&= w_{o,i} \left(\nabla_{\vec{w}_{h,i}} \widehat{d}_{o,i}\left(t\right)\right),
\end{align}
since only $\widehat{d}_{o,i}\left(t\right)$ depends on $\vec{w}_{h,i}$. Thus, from Eq.~\eqref{eqApproxVecdo}, we get
\begin{align}
\nabla_{\vec{w}_{h,i}}V_o\left(t\right)		&\approx w_{o,i} \left(\nabla_{\vec{w}_{h,i}} c_{o,i}\left(t\right) * \widehat{h}\left(t\right)\right).
\label{eqNablaVo}
\end{align}
We know that $c_{o,i}\left(t\right) = \sum_{s} \alpha \left(t-t_{h,i}^{s} \right)$, where $t_{h,i}^{s}$ denotes the $s^{th}$ spiking instant of $i^{th}$ neuron in the hidden layer. Using the chain rule of differentiation, we get
\begin{align}
\nabla_{\vec{w}_{h,i}}  c_{o,i}\left(t\right) \approx \left( \sum_{s} \delta\left(t-t_{h,i}^s\right) \frac{\widehat{\vec{d}}_{h}\left(t_{h,i}^s\right)}{V_{h,i}'\left(t_{h,i}^s\right)} \right)* \alpha ' \left(t\right).
\label{eqGradcoi}
\end{align}
Refer to the appendix \ref{app_grad_approx} for a detalied derivation of Eq.~\eqref{eqGradcoi}.
Using Eq.~\eqref{eqNablaVo} and \eqref{eqGradcoi}, we obtain an approximation to $\nabla_{\vec{w}_{h,i}}V_o\left(t\right)$ as
\begin{align}
\nabla_{\vec{w}_{h,i}} V_o\left(t\right)    		&\approx w_{o,i} \cdot \left( \sum_{s}\delta\left(t-t_{h,i}^s\right) \frac{\widehat{\vec{d}}_{h}\left(t_{h,i}^s\right)}{V_{h,i}'\left(t_{h,i}^s\right)} \right)* \left(\alpha ' \left(t\right)* \widehat{h}\left(t\right) \right).
\label{eqNablaVo2}
\end{align}

\color{black}
Note that the key enabling idea in the derivation of the above learning rule is the use of the inverse of the time rate of change of the neuronal membrane potential to capture the dependency of its spike time on its membrane potential, as shown in the appendix~\ref{app_grad_approx} in detail.

\color{black}

\subsection{Spatio-Temporal Error Backpropagation}
Incorporating the approximation from Eq.~\eqref{eqGradWoVo} into Eq.~\eqref{eqOutNormed}, we get the weight adaptation rule for $\vec{w}_{o}$ as
\begin{align}
\Delta \vec{w}_{o} = r_{o}\int_{0}^T \frac{ \widehat{\vec{d}}_{o}\left(t\right)}{\|  \widehat{\vec{d}}_{o}\left(t\right)\|} e\left(t\right)dt.
\label{del_wo}
\end{align}
Similarly incorporating the approximation made in Eq.~\eqref{eqNablaVo2} into Eq.~\eqref{eqMultilayerNormed}, we obtain the weight adaptation rule for $\vec{w}_{h,i}$ as
\begin{align}
\Delta \vec{w}_{h,i} &= r_h \cdot w_{o,i} \cdot \int_{t=0}^T \left( \left( \sum_{s} \delta\left(t-t_{h,i}^s\right) \frac{\widehat{\vec{d}}_{h}\left(t_{h,i}^s\right)}{V_{h,i}'\left(t_{h,i}^s\right)} \right)* \alpha ' \left(t\right)* \widehat{h}\left(t\right) \right)  e\left(t\right) dt.
\label{eqBackPdw}
\end{align}
Thus the adaptation rule for the weight matrix $W_{h}$ is given as
\begin{align}
\Delta W_{h} &= r_h \cdot \int_{t=0}^T \left( \left( U_{h}(t) \vec{w}_{o} \widehat{\vec{d}}_{h}^{T}\left(t\right) \right)* \alpha ' \left(t\right)* \widehat{h}\left(t\right) \right)  e\left(t\right) dt,
\label{eqBackPdw1}
\end{align}
where $U_{h}(t)$ is a $m \times m$ diagonal matrix with $i^{th}$ diagonal entry given as
\begin{align}
u_{h,ii}(t) = \frac{ \sum_{s}\delta\left(t-t_{h,i}^s\right)}{V_{h,i}'\left(t\right)}.
\end{align}
\color{\highlightAdd}
Note that Eq.~\eqref{eqBackPdw1} requires $\bigO (mn)$ convolutions to compute $\Delta W_{h}$. Using the identity (derived in appendix \ref{app_conv_proof})
\begin{align}
\int_{t} \left( x\left(t\right) * y\left(t\right) \right) z\left(t\right) dt = \int_{t} \left( z\left(t\right) * y\left(-t\right) \right) x\left(t\right) dt,
\end{align}
\color{black}
equation~\eqref{eqBackPdw1} can be equivalently written in following  form, which lends itself to a more efficient implementation involving only $\bigO (1)$ convolutions.
\begin{align}
\Delta W_{h} &= r_h \cdot \int_{t=0}^T \left( e\left(t\right) * \alpha ' \left(-t\right)* \widehat{h}\left(-t\right) \right) U_{h}(t) \vec{w}_{o} \widehat{\vec{d}}_{h}^{T} \left(t\right) dt.
\label{eqBackPdw2}
\end{align}
\color{\highlightAdd}
Rearranging the terms as follows brings forth the inherent process of \emph{spatio-temporal backpropagation of error} happening during NormAD based training.
\begin{align}
    \Delta W_{h} &= r_h \cdot \int_{t=0}^T U_{h}(t) \left( \left( \vec{w}_{o} e\left(t\right) \right) * \alpha ' \left(-t\right)* \widehat{h}\left(-t\right) \right) \widehat{\vec{d}}_{h}^{T} \left(t\right) dt.
\end{align}
Here spatial backpropagation is done through the weight vector $\vec{w}_{o}$ as 
\begin{equation}
\vec{e}^{spat}_{h}(t) = \vec{w}_{o} e(t)
\end{equation}
and then temporal backpropagation by convolution with time reversed kernels $\alpha ' (t)$ and $\widehat{h} (t)$ and sampling with $U_{h}(t)$ as
\begin{equation}
\vec{e}^{temp}_{h}(t) =
U_{h}(t) \left(\vec{e}^{spat}_{h}(t) * \alpha ' \left(-t\right) * \widehat{h}\left(-t\right)\right).
\end{equation}
\color{black}
\highlightDel{In addition Eq.~\eqref{eqBackPdw2} also brings forth the intuition that NormAD based training of multi-layer SNNs is actually \emph{spatio-temporal error backpropagation}, where spatial backpropagation is done through the weight vector $\vec{w}_{o}$ and temporal backpropagation using time reversed kernels $\alpha ' (t)$ and $\widehat{h} (t)$.}
It will be more evident when we generalize it to SNNs with arbitrarily many hidden layers.

From Eq.~\eqref{eqBackPdw}, note that the weight update for synapses of a neuron in hidden layer depends on its own spiking activity thus suggesting the spike-induced nature of weight update. However, in case all the spikes of the hidden layer vanish in a particular training iteration, there will be no spiking activity in the output layer and as per  Eq.~\eqref{eqBackPdw} the weight update $\Delta \vec{w}_{h,i}=\vec{0}$ for all subsequent iterations. To avoid this, regularization techniques such as constraining the average spike rate of neurons in the hidden layer to a certain range can be used, though it has not been used in the present work.

\begin{figure}[!h]
    \includegraphics[scale=0.42]{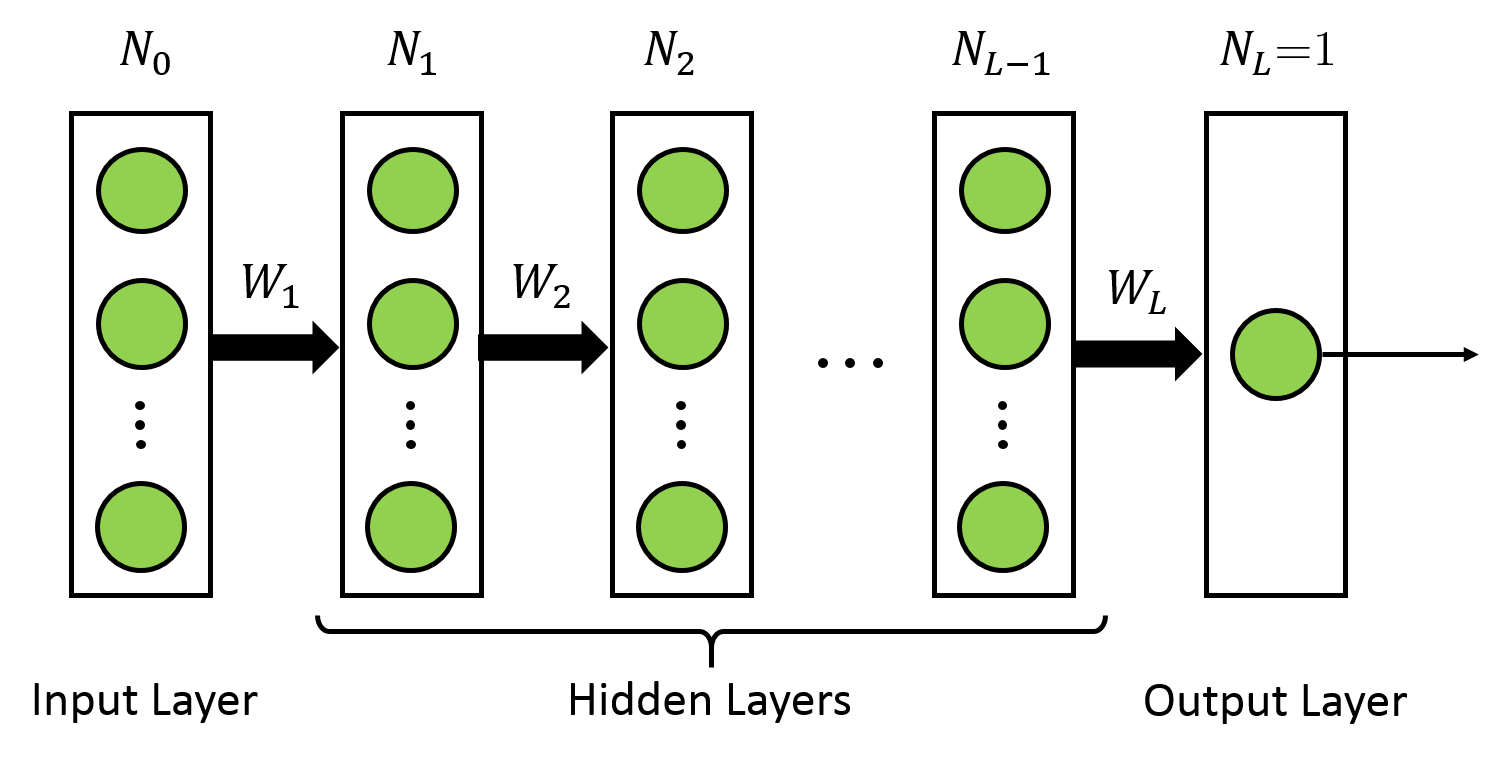}
    \caption{Fully connected feedforward SNN with $L$ layers ($N_{0} \rightarrow N_{1} \rightarrow N_{2} \cdots N_{L-1} \rightarrow 1$)}
    \label{figMultiLayerSNN_cartoon}
\end{figure}

\subsubsection{Generalization to Deep SNNs}
\label{spatTempBackProp}
For the case of feedforward SNNs with two or more hidden layers, the weight update rule for output layer remains the same as in Eq. \eqref{del_wo}.
Here, we provide the general weight update rule for any particular hidden layer of an arbitrary fully connected feedforward SNN $N_{0} \rightarrow N_{1} \rightarrow N_{2} \cdots N_{L-1} \rightarrow 1$ with L layers as shown in Fig. \ref{figMultiLayerSNN_cartoon}.
This can be obtained by the straight-forward extension of the derivation for the case with single hidden layer discussed above.
For this discussion, the subscript \textit{h} or \textit{o} indicating the layer of the corresponding neuron in the previous discussion is replaced by the  layer index to accommodate arbitrary number of layers.
The iterative weight update rule for synapses connecting neurons in layer $l-1$ to neurons in layer $l$ viz., $W_{l}$ $(0 < l < L)$ is given as follows:
\begin{align}
\label{eqSptTmpBackP}
\Delta W_{l} = r_h \int_{t=0}^{T} \vec{e}^{temp}_{l}(t) \vec{\widehat{d}}_{l}^{T}(t) dt \qquad \text{for } 0<l<L,
\end{align}
where
\begin{equation}
\vec{e}^{temp}_{l}(t) =
 \begin{cases}
U_{l}(t) \left(\vec{e}^{spat}_{l}(t) * \alpha ' \left(-t\right) * \widehat{h}\left(-t\right)\right)		&  1<l<L\\
   e(t)		& l = L,
 \end{cases}
 \label{eqTemporalBackProp}
\end{equation}
performs \emph{temporal backpropagation} following the \emph{spatial backpropagation} as
\begin{equation}
\vec{e}^{spat}_{l}(t) = W_{l+1}^{T} \vec{e}^{temp}_{l+1}(t) \qquad \text{for } 1<l<L.
\label{eqSpaialBackProp}
\end{equation}
Here $U_{l}(t)$ is an $N_{l} \times N_{l}$ diagonal matrix with $n^{th}$ diagonal entry given as
\begin{align}
\label{eqInvDiffV}
u_{l,nn}(t) = \frac{ \sum_{s}\delta\left(t-t_{l,n}^s\right)}{V_{l,n}'\left(t\right)},
\end{align}
where $V_{l,n}\left(t\right)$ is the membrane potential of $n^{th}$ neuron in layer $l$ and $t_{l,n}^s$ is the time of its $s^{th}$ spike.
\color{\highlightAdd}
From Eq. \eqref{eqTemporalBackProp}, note that temporal backpropagation through layer $l$ requires $\bigO\left(N_{l}\right)$ convolutions.
\color{black}

\section{Numerical validation}
\label{validate}
In this section we validate the applicability of NormAD based spatio-temporal error backpropagation to the training of multi-layer SNNs. The algorithm comprises of Eq.~\eqref{eqSptTmpBackP} - \eqref{eqInvDiffV}.
%

\subsection{XOR Problem}
\label{XORtraining}
XOR problem is a prominent example of non-linear classification problems  which can not be solved using the single layer neural network architecture and hence compulsorily require a multi-layer network. Here, we present how proposed NormAD based training was employed to solve a spike domain formulation of the XOR problem for a multi-layer SNN. The XOR problem is similar to the one used in \cite{bohte2002error} and represented by Table~\ref{tableXOR}. There are $3$ input neurons and $4$ different input spike patterns given in the $4$ rows of the table, where temporal encoding is used to represent logical $0$ and $1$. The numbers in the table represent the arrival time of spikes at the corresponding neurons. The bias input neuron always spikes at $t=0\,$ms. The other two inputs can have two types of spiking activity viz., presence or absence of a spike at $t=6\,$ms, representing logical $1$ and $0$ respectively. The desired output is coded such that an early spike (at $t=10\,$ms) represents a logical $1$ and a late spike (at $t=16\,$ms) represents a logical $0$. 
\begin{table}[!h]
	\renewcommand{\arraystretch}{1.3}
	\caption{XOR Problem set-up from \cite{bohte2002error}, which uses arrival time of spike to encode logical $0$ and $1$.}
	\label{tableXOR}
	\centering
	\begin{tabular}{c c c | c}
		\hline
		\multicolumn{3}{c|}{Input spike time (ms)} & \multirow{2}{*}{Output }\\
		Bias & Input $1$ & Input $2$ &spike time (ms) \\
		\hline
		0 & - & - & 16\\
		0 & - & 6 & 10\\
		0 & 6 & - & 10\\
		0 & 6 & 6 & 16\\
		\hline
	\end{tabular}
\end{table}

In the network reported in \cite{bohte2002error}, the three input neurons had $16$ synapses with axonal delays of $0,1,2,...,15\,$ms respectively. Instead of having multiple synapses we use a set of $18$ different input neurons for each of the three inputs such that when the first neuron of the set spikes, second one spikes after $1\,$ms, third one after another $1\,$ms and so on. Thus, there are $54$ input neurons comprising of three sets with $18$ neurons in each set. So, a $54 \rightarrow 54 \rightarrow 1$ feedforward SNN is trained to perform the XOR operation in our implementation. Input spike rasters corresponding to the $4$ input patterns are shown in Fig.~\ref{figXorRaster} (left).
\begin{figure}[!h]
	\includegraphics[scale=0.7]{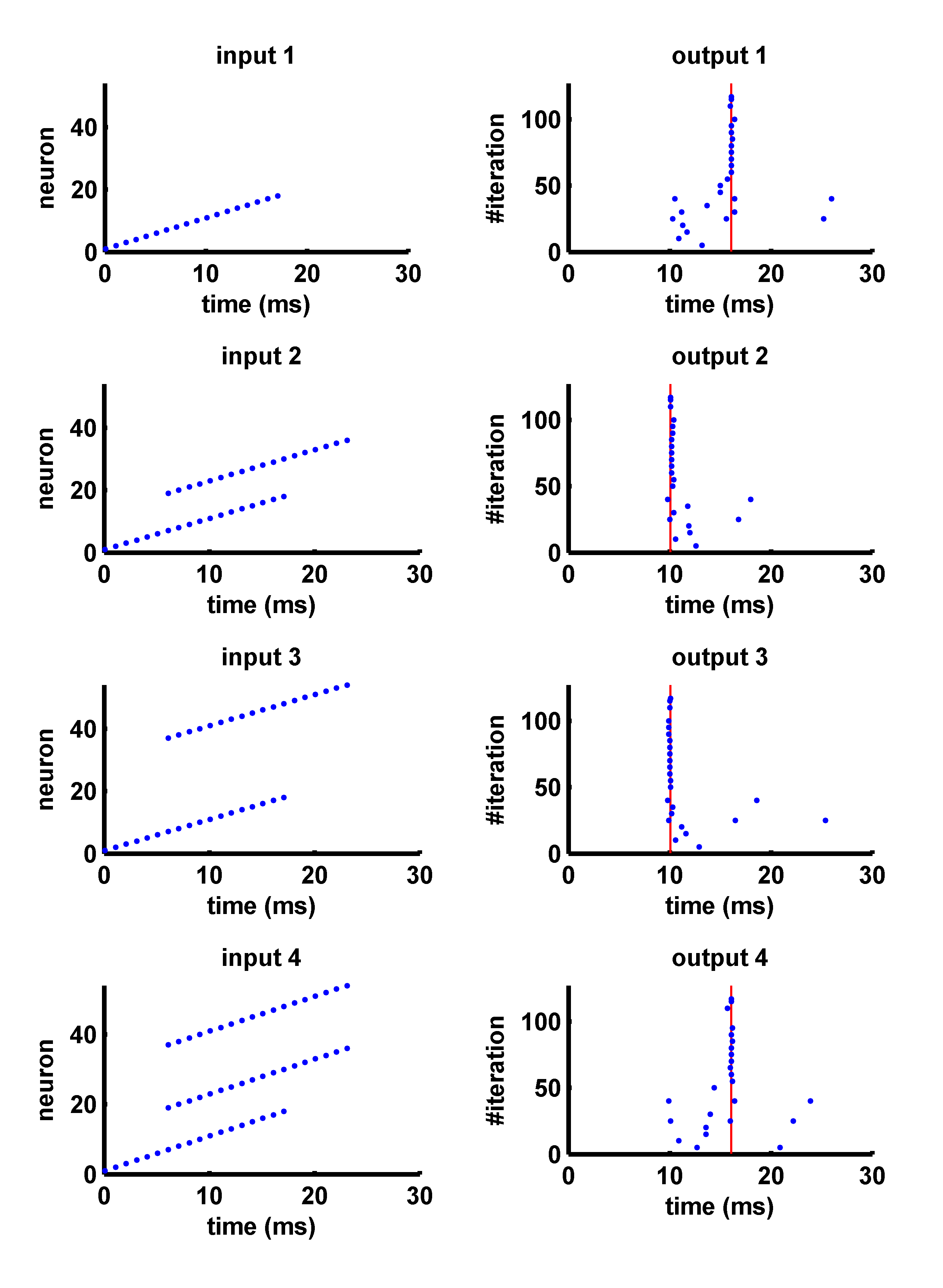}
	\caption{XOR problem: Input spike raster  (left) and corresponding output spike raster (right - blue dots) obtained during NormAD based training of a $54 \rightarrow 54 \rightarrow 1$ SNN with vertical red lines marking the position of desired spikes. The output spike raster is plotted for one in every 5 training iterations for clarity.}
	\label{figXorRaster}
\end{figure}

Weights of synapses from the input layer to the hidden layer were initialized randomly using Gaussian distribution, with $80\%$ of the synapses having positive mean weight (excitatory) and rest $20\%$ of the synapses having negative mean weight (inhibitory). The network was trained using NormAD based  spatio-temporal error backpropagation. Figure~\ref{figXorRaster} plots the output spike raster (on right) corresponding to each of the four input patterns (on left), for an exemplary initialization of the weights from the input to the hidden layer. 
As can be seen, convergence was achieved in less than $120$ training iterations in this experiment.

The necessity of a multi-layer SNN for solving an XOR problem is well known, but to demonstrate the effectiveness of NormAD based training to hidden layers as well, we conducted two experiments. For $100$ independent random initializations of the synaptic weights to the hidden layer, the SNN was trained with (i) non-plastic hidden layer, and  (ii) plastic hidden layer. The output layer was trained using Eq.~\eqref{del_wo} in both the experiments.
Figures~\ref{meanCorrXor} and \ref{stdDevCorrXor} show the mean and standard deviation respectively of spike correlation against training iteration number for the two experiments.
For the case with non-plastic hidden layer, the mean correlation reached close to 1, but the non-zero standard deviation represents a sizable number of experiments which did not converge even after $800$ training iterations.
When the synapses in hidden layer  were also trained,  convergence was obtained for all the $100$ initializations within $400$ training iterations. The convergence criteria used in these experiments was to reach the perfect spike correlation metric of $1.0$.

\begin{figure}[!h]
    \centering
    \begin{subfigure}[b]{\textwidth}
		\centering
        \includegraphics[scale=0.6]{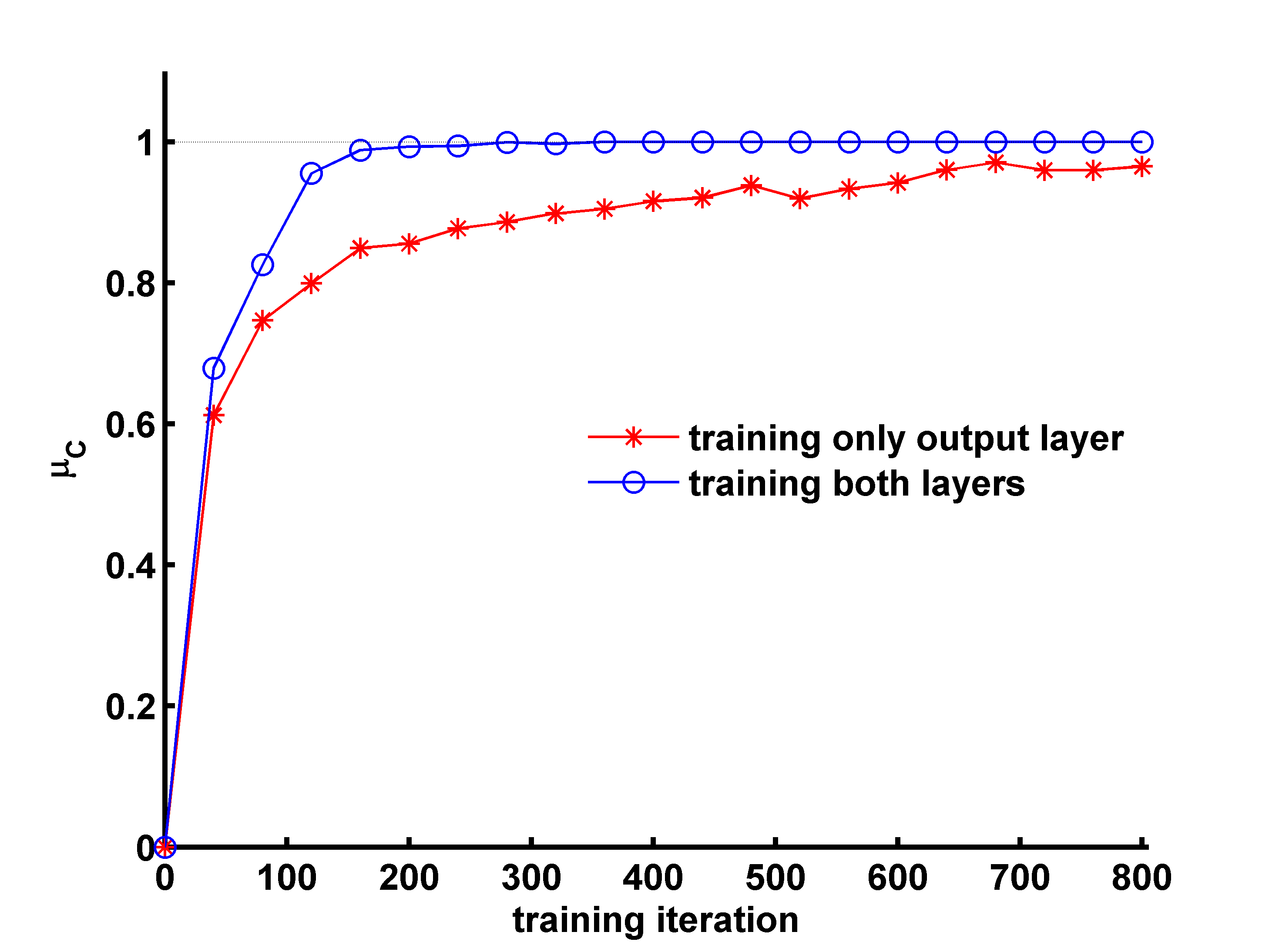}
        \caption{Mean spike correlation \label{meanCorrXor}}
    \end{subfigure}
    \begin{subfigure}[b]{\textwidth}
		\centering
        \includegraphics[scale=0.6]{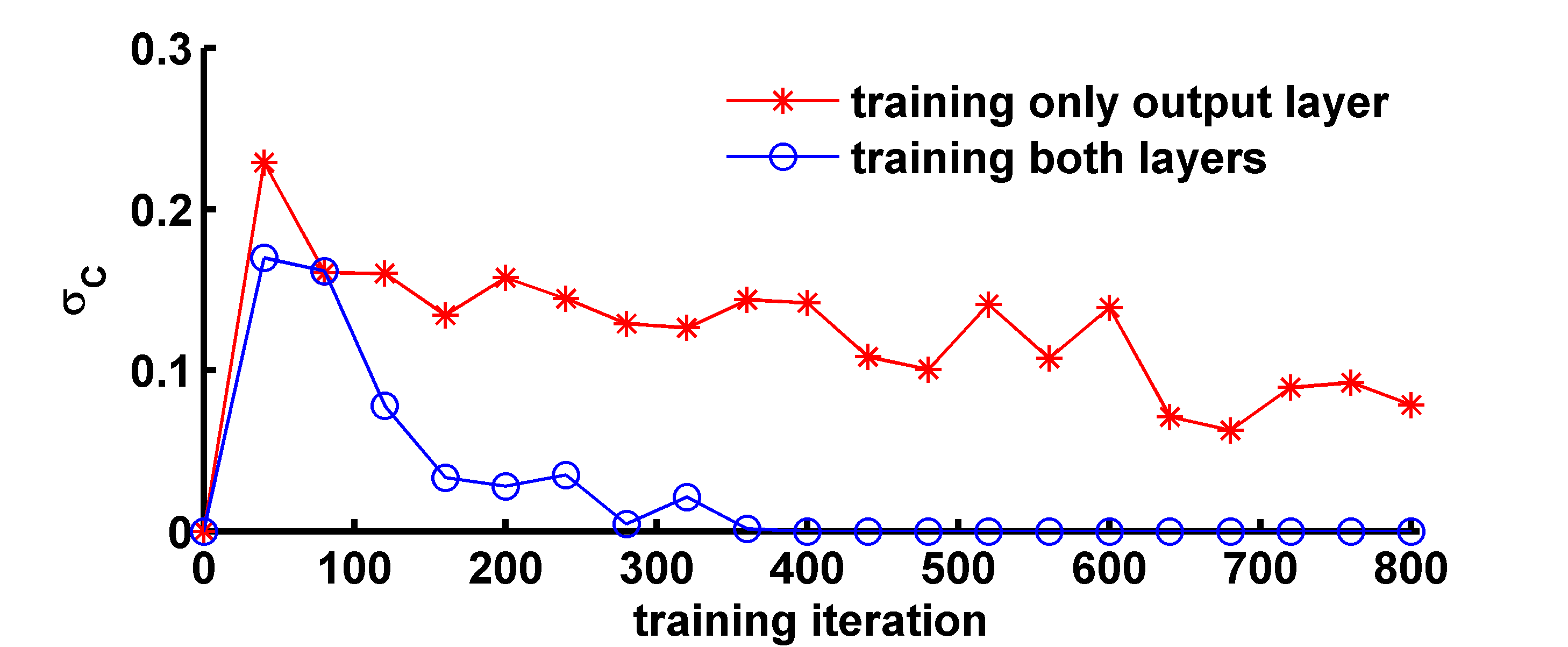}
        \caption{Standard deviation of spike correlation \label{stdDevCorrXor}}
    \end{subfigure}    
	\caption{Plots of \subref{meanCorrXor} mean and \subref{stdDevCorrXor} standard deviation of spike correlation metric over $100$ different initializations of $54 \rightarrow 54 \rightarrow 1$ SNN, trained for the XOR problem with non-plastic hidden layer (red asterisk) and plastic hidden layer (blue circles).}
	\label{figXorCorr}
\end{figure}

\subsection{Training SNNs with 2 Hidden Layers}
\label{multilayerTraining}
Next, to demonstrate spatio-temporal error backpropagation through multiple hidden layers, we applied the algorithm to train $100 \rightarrow 50 \rightarrow 25 \rightarrow 1$ feedforward SNNs for general spike based training problems. The weights of synapses feeding the output layer were initialized to $0$, while synapses feeding the hidden layers were initialized using a uniform random distribution and with $80\%$ of them excitatory and the rest $20\%$ inhibitory. Each training problem comprised of $n=100$ input spike trains and one desired output spike train, all generated to have Poisson distributed spikes with arrival rate $20\,$s$^{-1}$ for inputs and $10\,$s$^{-1}$ for the output, over an epoch duration $T=500\,$ms. Figure~\ref{figMultilayerRaster} shows the progress of training for an exemplary training problem by plotting the output spike rasters for various training iterations overlaid on plots of vertical red lines denoting the positions of desired spikes.

\begin{figure}[!h]
	\includegraphics[scale=0.5]{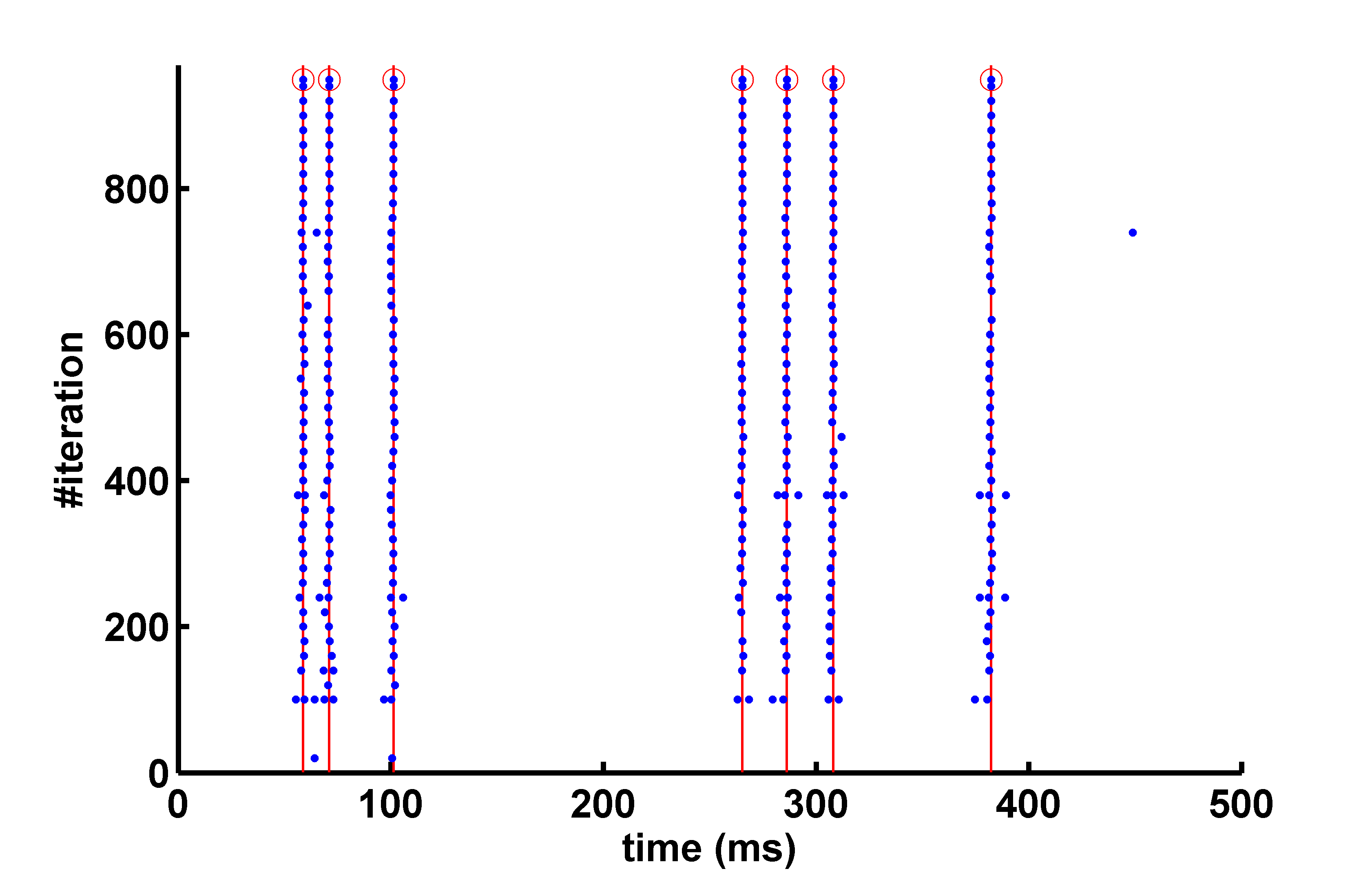}
	\caption{Illustrating NormAD based training of an exemplary problem for 3-layer $100 \rightarrow 50 \rightarrow 25 \rightarrow 1$ SNN. The output spike rasters (blue dots) obtained during one in every 20 training iterations (for clarity) is shown, overlaid on plots of vertical red lines marking positions of the desired spikes.}
	\label{figMultilayerRaster}
\end{figure}

To assess the gain of training hidden layers using NormAD based spatio-temporal error backpropagation, we ran a set of $3$ experiments. For $100$ different training problems for the same SNN architecture as described above, we studied the effect of (i) training only the output layer weights, (ii) 
training only the outer 2 layers and (iii) training all the $3$ layers.
\begin{figure}[!h]
	\centering
	\includegraphics[scale=0.6]{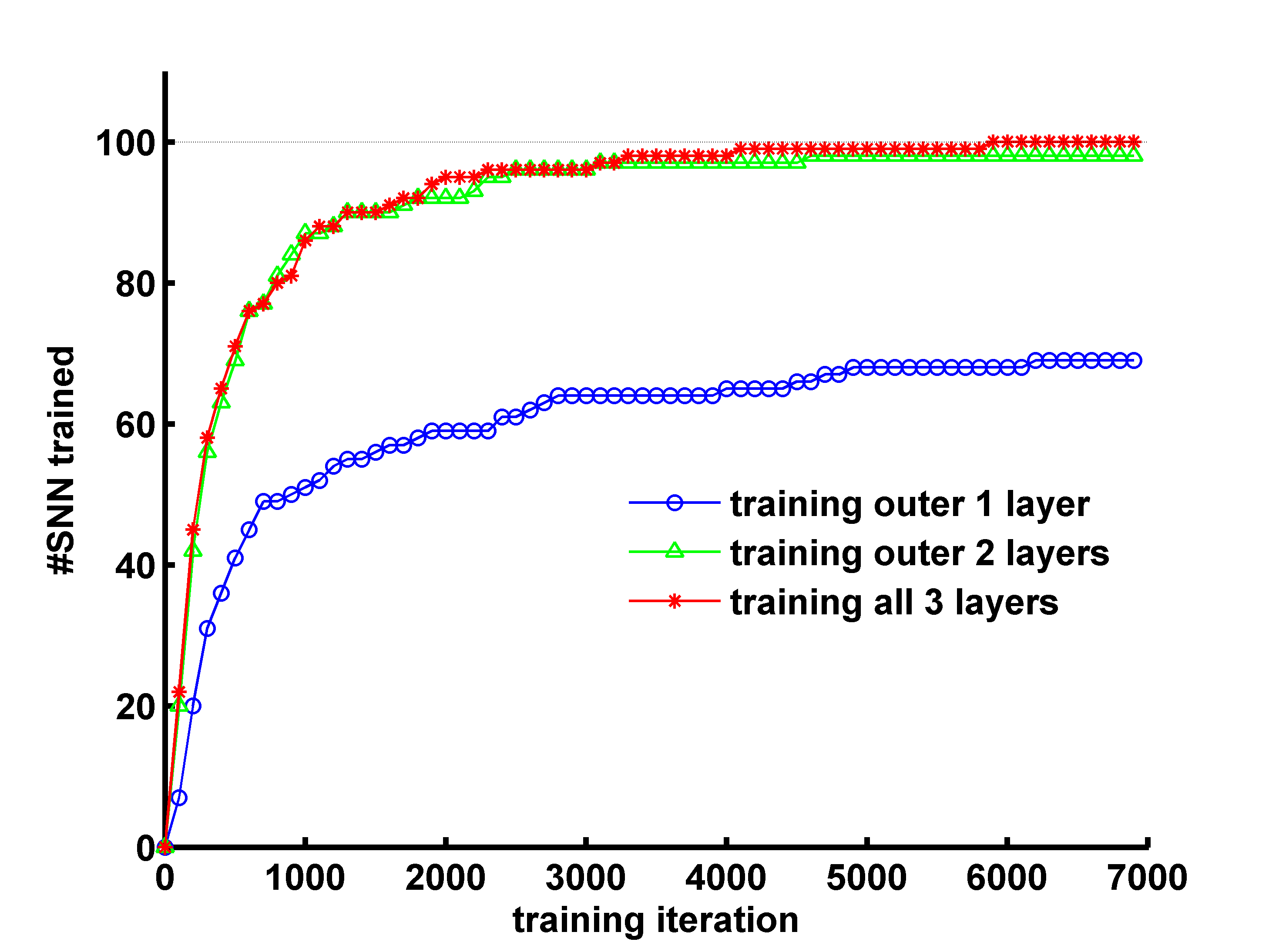}
	\caption{Plots showing cumulative number of training problems for which convergence was achieved out of total $100$ different training problems for 3-layer $100 \rightarrow 50 \rightarrow 25 \rightarrow 1$ SNNs.}
    \label{figHistMultiLayerCumulativeTrained}
\end{figure}
Figure~\ref{figHistMultiLayerCumulativeTrained} plots the cumulative number of SNNs trained against number of training itertions for the $3$ cases, where the criteria for completion of training is reaching the correlation metric of $0.98$ or above.
Figures~\ref{meanCorr} and \ref{stdDevCorr} show plots of mean and standard deviation respectively of spike correlation against training iteration number 
for the $3$ experiments. As can be seen, in the third experiment when all $3$ layers were trained, all $100$ training problems converged within $6000$ training iterations. In contrast, the first $2$ experiments have non-zero standard deviation even until $10000$ training iterations indicating non-convergence for some of the cases. In the first eperiment, where only synapses feeding the output layer were trained, convergence was achieved only for $71$ out of $100$ training problems after $10000$ iterations. However, when the synapses feeding the top two layers or all three layers were trained, the number of cases reaching convergenvce rose to $98$ and $100$ respectively, thus proving the effectiveness of the proposed NormAD based training  method for multi-layer SNNs.

\begin{figure}[!h]
	\centering
    \begin{subfigure}[b]{\textwidth}
		\centering
        \includegraphics[scale=0.6]{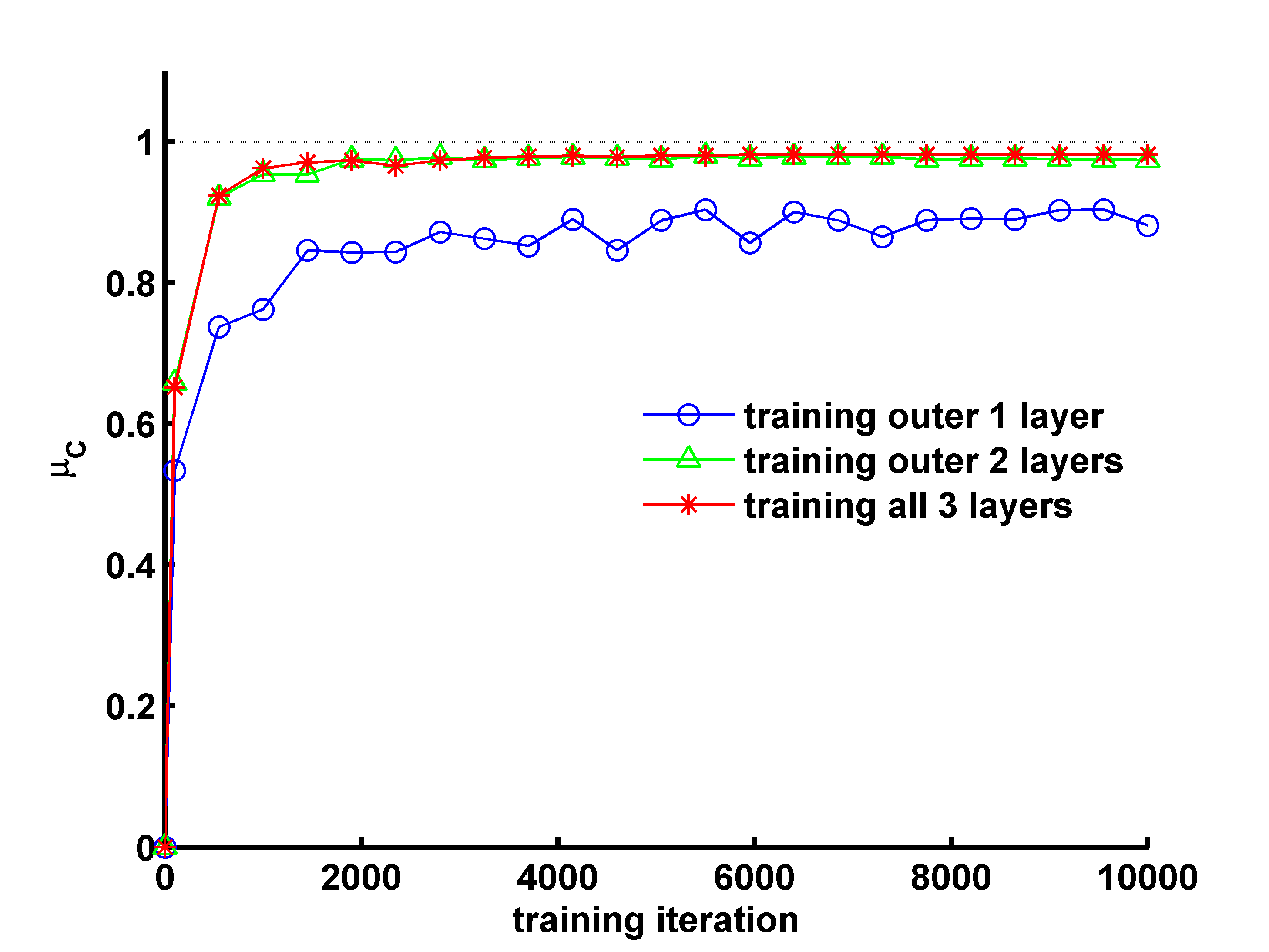}
        \caption{Mean spike correlation \label{meanCorr}}
    \end{subfigure}
    \begin{subfigure}[b]{\textwidth}
		\centering
        \includegraphics[scale=0.6]{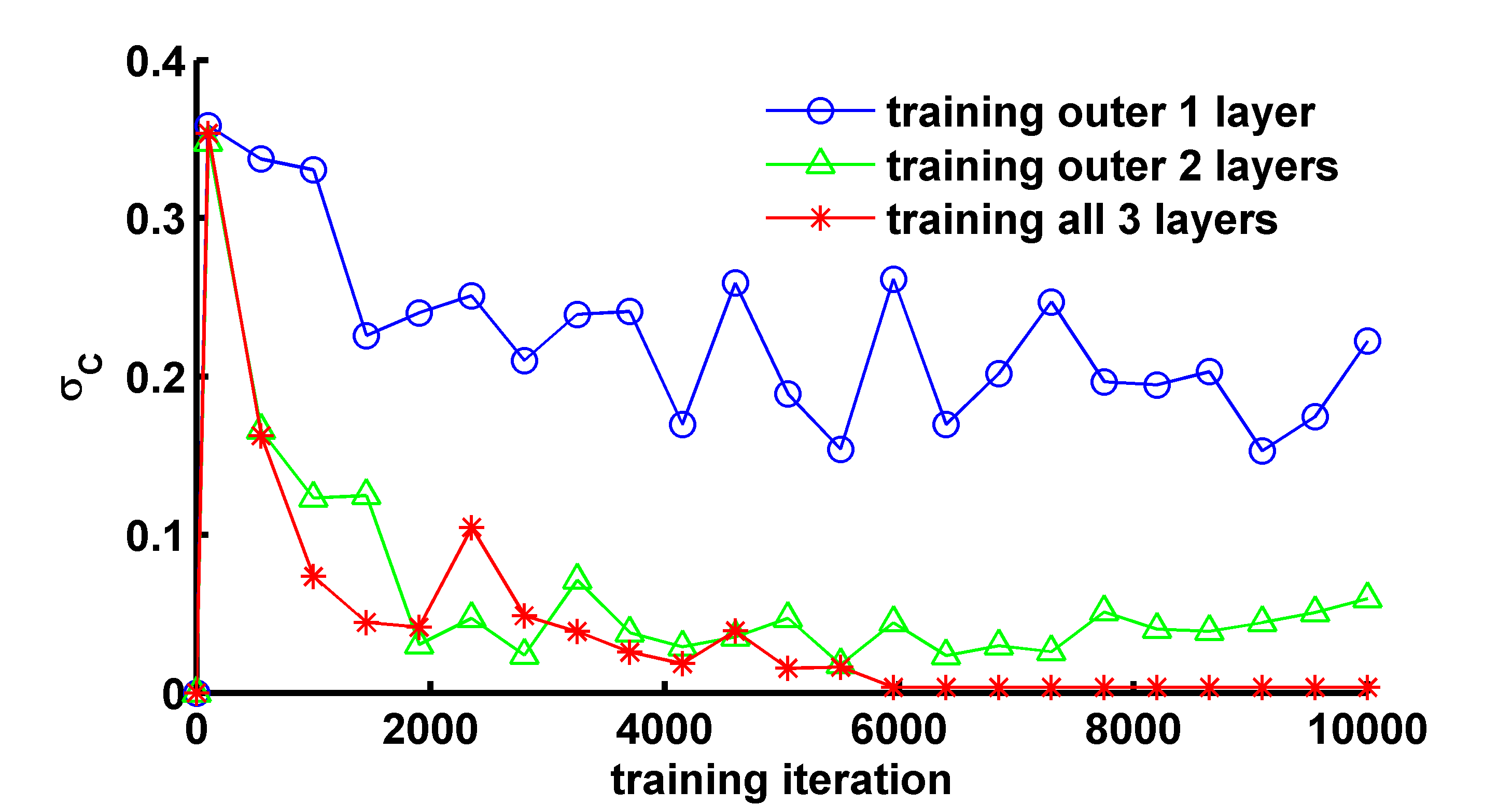}
        \caption{Standard deviation of spike correlation\label{stdDevCorr}}
    \end{subfigure} 
	\caption{Plots of \subref{meanCorr} mean and \subref{stdDevCorr} standard deviation of spike correlation metric while partially or completely training 3-layer $100 \rightarrow 50 \rightarrow 25 \rightarrow 1$ SNNs for $100$ different training problems.}
	\label{figCorrMultiLayerCmp1}
\end{figure}

\section{Conclusion}
\label{secConclusion}
We developed NormAD based spaio-temporal error backpropagation to train multi-layer feedforward spiking neural networks. It is the spike domain analogue of error backpropagation algorithm used in second generation neural networks. The derivation was accomplished by first formulating the corresponding training problem as a non-convex optimization problem and then employing Normalized Approximate Descent based optimization to obtain the weight adaptation rule for the SNN.
The learning rule was validated by applying it to train $2$ and $3$-layer feedforward SNNs for a spike domain formulation of the XOR problem and general spike domain training problems respectively. 

The main contribution of this work is hence the development of a learning rule for spiking neural networks with arbitrary number of  hidden layers. One of the major hurdles in achieving this has been the problem of backpropagating errors through non-linear leaky integrate-and-fire dynamics of a spiking neuron. We have tackled this by introducing temporal error backpropagation and quantifying the dependence of the time of a spike on the corresponding membrane potential by the inverse temporal rate of change of the membrane potential. This together with the spatial backpropagation of errors constitutes NormAD based training of multi-layer SNNs.
\color{black}

The  problem of local convergence while training second generation deep neural networks is tackled by unsupervised \emph{pretraining} prior to the application of error backpropagation \cite{hinton2006fast,Erhan:2010:WUP:1756006.1756025}. Development of such unsupervised pretraining techniques for deep SNNs is a topic of future research, as NormAD could be applied in principle to develop SNN based autoencoders.
 
\appendix
\numberwithin{equation}{section}
\appendixpage
\section{Gradient Approximation}
\label{app_grad_approx}
 
\color{black}
Derivation of Eq.~\ref{eqGradcoi} is presented below:
\begin{align}
\nabla_{\vec{w}_{h,i}}  c_{o,i}\left(t\right) 
\nonumber						   &= \sum_{s} \frac{\partial \alpha \left(t-t_{h,i}^s \right)}{\partial t_{h,i}^s} \cdot \nabla_{\vec{w}_{h,i}} t_{h,i}^s 	\qquad (\text{from Eq. \eqref{eq_c_oi_of_t}})\\
								   &= \sum_{s} - \alpha ' \left(t-t_{h,i}^s \right) \cdot \nabla_{\vec{w}_{h,i}} t_{h,i}^s
\label{eqGradcoiOrigin}
\end{align}
To compute $\nabla_{\vec{w}_{h,i}} t_{h,i}^s$, let us assume that a small change $\delta w_{h,ij}$ in $w_{h,ij}$ led to changes in $V_{h,i}(t)$ and  $t_{h,i}^s$ by $\delta V_{h,i}(t)$ and  $\delta t_{h,i}^s$ respectively i.e.,
\begin{align}
V_{h,i}(t_{h,i}^s + \delta t_{h,i}^s) + \delta V_{h,i}(t_{h,i}^s + \delta t_{h,i}^s) = V_{T}.
\label{eq_delta_spike}
\end{align}
From Eq. \eqref{eqApproxVhi}, $\delta V_{h,i}(t)$ can be approximated as
\begin{align}
\delta V_{h,i}(t) \approx \delta w_{h,ij} \cdot \widehat{d}_{h,j}(t),
\end{align}
hence from Eq. \eqref{eq_delta_spike} above
\begin{align}
\nonumber V_{h,i}(t_{h,i}^s) + \delta t_{h,i}^s V_{h,i}'(t_{h,i}^s) + \delta w_{h,ij} \cdot \widehat{d}_{h,j}(t_{h,i}^s + \delta t_{h,i}^s) &\approx V_{T}
\end{align}
\begin{align}
& \nonumber \implies \frac{\delta t_{h,i}^s}{\delta w_{h,ij}} \approx \frac{-\widehat{d}_{h,j}(t_{h,i}^s + \delta t_{h,i}^s)}{V_{h,i}'(t_{h,i}^s)} \qquad (\text{since } V_{h,i}(t_{h,i}^s) = V_{T}
)\\
& \nonumber \implies \frac{\partial t_{h,i}^s}{\partial w_{h,ij}} \approx \frac{-\widehat{d}_{h,j}(t_{h,i}^s)}{V_{h,i}'(t_{h,i}^s)}\\
& \implies \nabla_{\vec{w}_{h,i}} t_{h,i}^s \approx \frac{-\widehat{\vec{d}}_{h}\left(t_{h,i}^s\right)}{V_{h,i}'\left(t_{h,i}^s\right)}.
\label{eq_grad_sp_time}
\end{align}
Thus using Eq. \eqref{eq_grad_sp_time} in Eq. \eqref{eqGradcoiOrigin} we get
\begin{align}
\nabla_{\vec{w}_{h,i}}  c_{o,i}\left(t\right)
\nonumber					&\approx \sum_{s} \alpha ' \left(t-t_{h,i}^s \right)\frac{\widehat{\vec{d}}_{h}\left(t_{h,i}^s\right)}{V_{h,i}'\left(t_{h,i}^s\right)}\\
							&\approx \left( \sum_{s} \delta\left(t-t_{h,i}^s\right) \frac{\widehat{\vec{d}}_{h}\left(t_{h,i}^s\right)}{V_{h,i}'\left(t_{h,i}^s\right)} \right)* \alpha ' \left(t\right).
\label{eqGradcoiAppendix}
\end{align}

Note that approximation in Eq.~\eqref{eq_grad_sp_time} is an important step towards obtaining weight adaptation rule for hidden layers, as it now allows us to approximately model the dependence of the spiking instant of a neuron on its inputs using the inverse of the time derivative of its membrane potential.

\color{\highlightAdd}
\section{}
\label{app_conv_proof}
\begin{lemma}
Given 3 functions $x(t)$, $y(t)$ and $z(t)$
\begin{align*}
\int_{t} \left( x\left(t\right) * y\left(t\right) \right) z\left(t\right) dt = \int_{t} \left( z\left(t\right) * y\left(-t\right) \right) x\left(t\right) dt.
\end{align*}
\end{lemma}
\begin{proof}
By definition of linear convolution
\begin{align*}
\int_{t} \left( x\left(t\right) * y\left(t\right) \right) z\left(t\right) dt 
&= \int_{t} \left( \int_{u} x\left(u\right) y\left(t - u\right) du \right) z\left(t\right) dt.
\end{align*}
Changing the order of integration, we get
\begin{align*}
\int_{t} \left( x\left(t\right) * y\left(t\right) \right) z\left(t\right) dt
&= \int_{u} x\left(u\right) \left( \int_{t}  y\left(t - u\right) z\left(t\right) dt \right)  du\\
&= \int_{u} x\left(u\right) \left( y\left(-u\right) * z\left(u\right) \right)  du\\
&= \int_{t} \left( z\left(t\right) * y\left(-t\right) \right) x\left(t\right)  dt.
\end{align*}
\end{proof}

\color{black}

{
\section*{Acknowledgment}
This research was  supported in part  by the U.S. National Science Foundation through the grant 1710009. 

The authors acknowledge the invaluable insights gained during their stay at Indian Institute of Technology, Bombay where the initial part of this work was conceived and conducted as a part of a master's thesis project.
We also acknowledge the reviewer comments which helped us expand the scope of this work and bring it to its present form.}

\end{document}